\documentclass[runningheads]{llncs}
\usepackage[T1]{fontenc}
\usepackage{graphicx}
\usepackage{booktabs}
\usepackage[misc]{ifsym}
\newcommand{\corr}{(\Letter)}
\usepackage{xcolor}
\usepackage{amsmath, amsthm, amssymb}
\usepackage{multirow}
\usepackage{enumitem}
\usepackage{bm}
\usepackage{algorithm}
\usepackage{algorithmic}
\usepackage{makecell}
\newtheorem{assumption}{Assumption}

\newcommand{\E}{\mathbb{E}}
\newcommand{\R}{\mathbb{R}}

\newcommand{\calD}{\mathcal{D}}
\newcommand{\calL}{\mathcal{L}}
\newcommand{\calM}{\mathcal{M}}
\newcommand{\calS}{\mathcal{S}}
\newcommand{\calE}{\mathcal{E}}
\newcommand{\calG}{\mathcal{G}}

\newcommand{\calV}{\mathcal{V}}
\newcommand{\calH}{\mathcal{H}}
\newcommand{\calR}{\mathcal{R}}
\newcommand{\bG}{\bm{G}}
\DeclareMathOperator{\Cov}{Cov}

\newcommand{\bg}{\bm{g}}

\newcommand{\bz}{\bm{z}}
\newcommand{\bM}{\bm{M}}
\newcommand{\bW}{\bm{W}}
\newcommand{\bA}{\bm{A}}
\newcommand{\bB}{\bm{B}}
\newcommand{\bU}{\bm{U}}
\newcommand{\bV}{\bm{V}}
\newcommand{\bSigma}{\bm{\Sigma}}
\newcommand{\norm}[1]{\left\lVert #1 \right\rVert}
\newcommand{\inner}[2]{\left\langle #1, #2 \right\rangle}
\newcommand{\Var}{\operatorname{Var}}
\newcommand{\Tr}{\operatorname{Tr}}

\newcommand{\IGA}{\textsc{IGA}}
\newcommand{\ERM}{\textsc{ERM}}

\newcommand{\LoRA}{\textsc{LoRA}}

\definecolor{mathblue}{RGB}{31, 119, 180}
\definecolor{medgreen}{RGB}{44, 160, 44}
\definecolor{legorange}{RGB}{255, 127, 14}
\definecolor{scipurple}{RGB}{148, 103, 189}
\definecolor{igared}{RGB}{214, 39, 40}

\begin{document}

\title{Invariant Gradient Alignment for Robust Reasoning Distillation}

\titlerunning{Invariant Gradient Alignment}
\author{Zehua Cheng\inst{1}\corr \and Wei Dai\inst{2} \and Jiahao Sun\inst{2}}
\authorrunning{Z. Cheng et al.}
\institute{University of Oxford, Oxford, UK\\
\email{zehua.cheng@cs.ox.ac.uk}
\and
FLock.io}

\maketitle

\begin{abstract}
Large language models (LLMs) suffer from shortcut learning: they systematically fail on out-of-distribution (OOD) inputs whose semantic surface differs from training data, even when the logical structure is identical. This undermines knowledge distillation pipelines that transfer chain-of-thought reasoning to smaller students. We introduce Invariant Gradient Alignment (IGA), a training framework that aligns gradient updates across semantically diverse but logically isomorphic examples via three innovations: (i) \emph{Logical Isomer Sets}---groups of problems sharing identical logical structure across distinct semantic domains (mathematics, medicine, law, science); (ii) a differentiable \emph{Continuous Gradient Conflict Mask}, $\bM = \exp(-\tau \cdot \bV)$, that suppresses parameter dimensions with high cross-domain gradient variance while preserving invariant directions; and (iii) a truncated SVD projection of the masked gradient back onto the LoRA low-rank manifold, maintaining parameter efficiency throughout. Theoretically, IGA yields tighter OOD generalization bounds than ERM, scaling with the number of isomer domains, and converges at the standard SGD rate under mild regularity. Empirically, IGA outperforms eight baselines across four benchmarks with accuracy gains up to 14.3 pp over ERM-SFT and a Logical Consistency Score of 0.031 versus 0.142---a fourfold improvement in representational invariance.
\keywords{Knowledge Distillation \and Invariant Learning \and Gradient Masking \and Out-of-Distribution Generalization \and Chain-of-Thought Reasoning}
\end{abstract}

\section{Introduction}
\label{sec:intro}

The ability to reason—to chain together logical steps in order to reach a sound conclusion—has long been considered a hallmark of intelligence that separates genuine understanding from mere pattern matching. Recent work on large language ` models has produced systems that achieve remarkable accuracy on formal reasoning tasks \cite{openai2023gpt4}, and the community has invested heavily in distilling this capacity into smaller, deployable student models via chain-of-thought (CoT) supervision \cite{wei2022chain}. Yet a persistent and uncomfortable gap remains: models that perform impressively on held-out test sets drawn from the same distribution as training data frequently collapse when evaluated on problems whose logical content is identical but whose semantic clothing differs. A student that learns to solve arithmetic word problems using dollar amounts may fail entirely when the same arithmetic structure appears in medical dosage calculations; a model trained on contract-law reasoning may misfire on analogous arguments in environmental regulation. This is the signature of shortcut learning \cite{marcus2020next}, and it represents not merely a benchmark gap but a fundamental reliability failure in high-stakes deployment settings.

The challenge is not superficial. Standard fine-tuning optimizes the empirical risk on a fixed training distribution, which makes it inherently susceptible to latching onto statistical regularities that correlate with labels in training domains but do not reflect the underlying logical structure of problems \cite{rosenfeld2021risks}. Invariant Risk Minimization (IRM) \cite{arjovsky2019irm} and its variants such as V-REx \cite{krueger2021out} address this in principle by requiring predictors that perform consistently across environments, but they impose architectural or loss-level constraints that are difficult to scale to the billions-of-parameters regime of modern LLMs, and they do not inherently address the distillation setting where a teacher's reasoning trace serves as the training signal. Meanwhile, gradient-based multi-task methods such as PCGrad \cite{yu2020gradient}, SAND-mask \cite{shahtalebi2021sand}, and gradient surgery \cite{shi2021gradient} reduce inter-task interference but do not exploit the key structural property available in reasoning distillation: the existence of logically isomorphic problem groups whose gradients \emph{should} agree on invariant parameters and \emph{will} conflict on shortcut dimensions.

This paper introduces \textbf{Invariant Gradient Alignment (IGA)}, a framework that is specifically designed to exploit logical isomorphism for robust reasoning distillation. The central observation underlying IGA is that if a reasoning student has genuinely learned the logical structure of a problem, its gradient on that problem should point in the same direction in parameter space regardless of which semantic domain the problem is expressed in. Conversely, dimensions in which gradients disagree across isomorphic variants are precisely the dimensions being driven by domain-specific superficial features rather than logical content. IGA operationalizes this intuition through three tightly integrated components: (i) a data construction pipeline that generates \emph{Logical Isomer Sets} by prompting a teacher LLM to map seed logical graphs into four distinct semantic domains; (ii) an online gradient masking algorithm that computes per-dimension gradient variance across the isomer domains and constructs a smooth, differentiable conflict mask $\bM = \exp(-\tau \bV)$ that softly zeroes out conflicting dimensions; and (iii) a subspace projection step that applies the mask in a reconstructed full-rank parameter space and re-factorizes the result back into the low-rank LoRA factors $\bA$ and $\bB$, preserving adapter compactness throughout training.

The contributions of this paper are as follows:

\begin{enumerate}[leftmargin=*]
  \item \textbf{Logical Isomer Sets:} We introduce a novel data construction methodology that creates semantically diverse but logically isomorphic training groups, enabling systematic measurement and mitigation of domain-specific gradient conflicts (Section~\ref{sec:isomer}).

  \item \textbf{Continuous Gradient Conflict Masking:} We propose a smooth, differentiable gradient conflict mask $\bM = \exp(-\tau \bV_d)$ that continuously suppresses shortcut parameter dimensions as a function of cross-domain variance, generalizing discrete binary masking approaches such as AND-mask \cite{parascandolo2020learning} and SAND-mask \cite{shahtalebi2021sand} (Section~\ref{sec:iga}).

  \item \textbf{Invariant LoRA Projection:} We develop a parameter-efficient training procedure that applies gradient masking in the reconstructed full-rank space and reprojects back to LoRA factors via truncated SVD, maintaining adaptation rank and memory efficiency (Section~\ref{sec:lora_proj}).

  \item \textbf{Logical Consistency Score:} We introduce a new evaluation metric, the Logical Consistency Score (LCS), that measures the trace of the covariance matrix of hidden representations across isomorphic problem variants, providing a domain-invariance measure that complements task accuracy (Section~\ref{sec:lcs}).

  \item \textbf{Theoretical Analysis:} We prove that IGA's update rule yields an OOD generalization bound that is tighter than ERM by a factor that scales with the number of isomer domains $N$, and we establish convergence of the masked gradient iteration under standard smoothness and boundedness assumptions (Appendix~\ref{app:theory}).
\end{enumerate}

\section{Related Work}
\label{sec:related}

\subsection{Chain-of-Thought Distillation}

The chain-of-thought prompting paradigm \cite{wei2022chain} established that eliciting multi-step reasoning traces dramatically improves LLM performance on complex tasks. Subsequent work demonstrated that these reasoning traces can serve as training supervision to distill the teacher's reasoning capacity into smaller student models \cite{ho2022large}. Orca \cite{zhang2022orca} scaled this idea by distilling from GPT-4.5's rich explanation traces, while \cite{huang2022large} showed that LLMs can iteratively improve their own reasoning via self-generated chains. The common thread across this work is the use of teacher-generated CoT traces as a privileged training signal \cite{vapnik2009new}. However, all of these approaches optimize the student on a single distribution of problems, leaving the student vulnerable to OOD degradation. \cite{sun2024invariant} introduced invariant causal knowledge distillation, which seeks to transfer causal rather than correlational features, but does not address the specific challenge of gradient-level conflict across logically isomorphic domains. IGA directly addresses this gap by making gradient alignment across isomorphic variants an explicit objective of the distillation procedure.

\subsection{Invariant Risk Minimization and Domain Generalization}

The theoretical foundation for learning predictors that generalize across environments was laid by \cite{peters2016causal} through invariant causal prediction, and extended to neural networks by \cite{arjovsky2019irm} through Invariant Risk Minimization (IRM). IRM minimizes an objective that includes a regularizer penalizing the gradient of the loss with respect to a dummy classifier variable, encouraging the representation to be equally predictive across all environments. V-REx \cite{krueger2021out} replaces this with variance minimization over environment risks, which is more stable in practice. \cite{rosenfeld2021risks} provided a careful theoretical analysis of when IRM can fail to recover the correct invariant features, identifying conditions on the number of environments and the dimensionality of the spurious features. Distributionally robust optimization \cite{sagawa2020distributionally} offers a related perspective by optimizing worst-case performance over uncertainty sets. Domain generalization via gradient matching \cite{shi2021gradient} is perhaps the most directly related prior work: it aligns gradients across training domains to encourage learning of domain-invariant features. However, all of these methods were designed primarily for classification tasks in vision or tabular settings, and none have been adapted to the challenging regime of autoregressive language generation with LoRA-based adaptation. IGA specializes and extends these ideas to the reasoning distillation setting, replacing environment-level constraints with per-dimension gradient masking and coupling the invariance mechanism with low-rank parameter efficiency. Furthermore, IGA constructs its own environments (the Logical Isomer Sets) rather than assuming they are given, which is critical because natural reasoning datasets do not come pre-partitioned into logically isomorphic groups.

\subsection{Reasoning Robustness and OOD Generalization in LLMs}

The brittleness of LLM reasoning under distributional shift has been documented across multiple dimensions. Sensitivity to surface-form variation \cite{liang2022holistic} shows that rephrasing prompts can cause dramatic accuracy swings. Cross-domain transfer of mathematical reasoning \cite{cobbe2021gsm8k}\cite{hendrycks2021measuring} demonstrates that models trained on school-level word problems often fail on structurally similar problems in physics or engineering. Logical reasoning benchmarks such as ReClor \cite{yu2020reclor} and LogiQA 2.0 \cite{liu2021logiqa} have specifically probed the distinction between logical structure and semantic surface, confirming that current fine-tuning approaches do not reliably decouple the two. \cite{allen2023physics} provided a theoretical account of why in-context learning in transformers is susceptible to spurious correlations. \cite{ma2025general} recently showed that exposure to diverse domain instances during RL-based training improves generalization, which aligns with IGA's construction of logically isomorphic multi-domain training batches. \cite{yoshida2025robust} demonstrated that cross-domain gradient alignment in representation learning improves invariance on OOD benchmarks, providing empirical support for the gradient-level mechanism underlying IGA. Collectively, this body of work motivates a training objective that explicitly encourages domain-invariant gradient directions, which is precisely what IGA provides.

\section{Methodology}
\label{sec:methodology}
\begin{figure}[t]
  \centering
  \includegraphics[width=\linewidth]{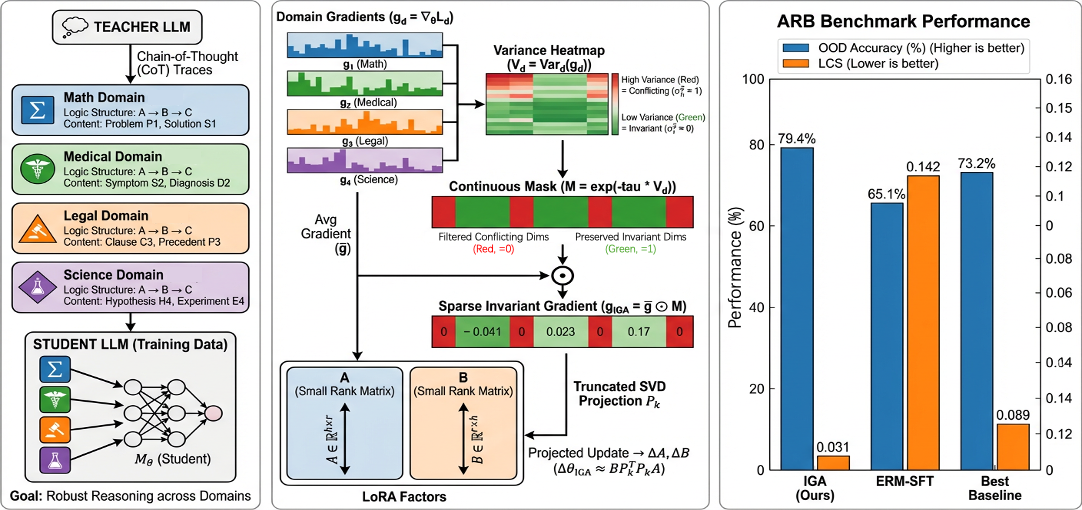}
  \caption{Overview of Invariant Gradient Alignment (IGA). \emph{Left:} A teacher LLM generates Chain-of-Thought traces for each domain instance in a Logical Isomer Set, providing training signal to the student. \emph{Center:} The IGA optimizer computes per-domain gradients, applies a continuous variance-based mask $\bM = \exp(-\tau \bV)$ to suppress conflicting shortcut dimensions, and projects the masked gradient back onto the LoRA manifold via truncated SVD. \emph{Right:} Performance comparison on the ARB benchmark showing that IGA achieves 79.4\% OOD accuracy versus 73.2\% for the best baseline (Pareto-GS) and 65.1\% for ERM-SFT, with a Logical Consistency Score of 0.031 versus 0.142 for ERM-SFT.}
  \label{fig:graphical_abstract}
\end{figure}
\subsection{Problem Formulation}
\label{sec:problem}

We consider knowledge distillation for chain-of-thought reasoning. Let $\mathcal{T}$ denote a teacher LLM and $\mathcal{S}_\theta$ a student parameterized by $\theta \in \R^P$. Under LoRA adaptation, $\theta = \theta_0 + \Delta\theta$ where $\theta_0$ is frozen and $\Delta\theta$ is parameterized by low-rank matrices $\{(\bA_\ell, \bB_\ell)\}_{\ell=1}^L$ with $\bA_\ell \in \R^{d \times r}$, $\bB_\ell \in \R^{r \times d}$, and $r \ll d$.

We posit $N$ semantic environments $\calE = \{e_1, \ldots, e_N\}$ (e.g., mathematics, medicine, law, science), each providing a distribution $\calD_{e_n}$ over input-output pairs $(x, y)$. The training distribution is $\calD = \sum_n \pi_n \calD_{e_n}$ with $\pi_n > 0$. The standard ERM objective minimizes:
\begin{equation}
  \mathcal{L}_{\ERM}(\theta) = \frac{1}{N} \sum_{n=1}^N \E_{(x,y) \sim \calD_{e_n}} \left[ \ell(\mathcal{S}_\theta(x), y) \right],
  \label{eq:erm}
\end{equation}
where $\ell$ is the token-level cross-entropy. ERM treats invariant and shortcut parameters identically; when $\calD_{\text{test}}$ differs from the training mixture, shortcut parameters receive large but misleading gradients, causing catastrophic OOD failure.

Decomposing $\theta = \theta^\star \oplus \theta^s$ into invariant and shortcut components, the shortcut gradient $\nabla_{\theta^s} \calL_{\ERM}$ may dominate under spurious correlations, yielding overfitting to the domain mixture. Our goal is to suppress updates along $\theta^s$ while concentrating learning on $\theta^\star$.

A \emph{logical isomer group} is a tuple $\calS_k = (x_k^{(1)}, y_k^{(1)}, \ldots, x_k^{(N)}, y_k^{(N)})$ whose $N$ instances share the same logical graph but differ in domain. The per-dimension gradient is:
\begin{equation}
  g_{k,n,d} = \frac{\partial \ell(\mathcal{S}_\theta(x_k^{(n)}), y_k^{(n)})}{\partial \theta_d}.
\end{equation}
A dimension $d$ is \emph{invariant} if $g_{k,n,d}$ is consistent across all $n$, and \emph{shortcut} if it exhibits high cross-domain variance. IGA suppresses shortcut dimensions and amplifies invariant ones through a learned masking function.

\subsection{Overall Architecture}
\label{sec:architecture}

Figure~\ref{fig:graphical_abstract} provides an overview. At each step, a sampled isomer group $\calS_k$ of $N$ logically isomorphic instances is processed in parallel to produce per-domain gradients $\bg_1, \ldots, \bg_N$. The IGA optimizer (i) reconstructs full-rank gradients from the LoRA factors, (ii) computes cross-domain variance $\bV_d$ per dimension, (iii) constructs the mask $\bM = \exp(-\tau \bV)$, (iv) applies $\bar{\bg} \odot \bM$, and (v) projects back to the LoRA manifold via truncated SVD. Updates are applied through AdamW (Algorithm~\ref{alg:iga}).

\subsection{Logical Isomer Set Construction}
\label{sec:isomer}

The masking quality depends on training groups that are genuinely logically isomorphic.

\textbf{Seed Problem Abstraction.} From a bank of 50,000 diverse reasoning problems, we use a teacher LLM (GPT-4.5) to extract each seed problem $p_k$'s abstract logical structure as a DAG $\calG_k = (V_k, E_k)$, where nodes are logical entities and edges are dependencies (implication, causation, conditional), stripping domain-specific content.

\textbf{Domain Instantiation.} Each DAG $\calG_k$ is instantiated in four semantic domains via structured prompts requiring (a) domain-appropriate vocabulary, (b) exact preservation of $\calG_k$'s dependencies, and (c) compatible chain-of-thought solution structure. The four domains are:
\begin{itemize}[leftmargin=*]
  \item \textbf{Mathematics} ($e_1$): Algebra, number theory, combinatorics, geometry.
  \item \textbf{Medical} ($e_2$): Diagnosis, pharmacology, case studies, epidemiology.
  \item \textbf{Legal} ($e_3$): Contract interpretation, statutory reasoning, case law, procedural analysis.
  \item \textbf{Science} ($e_4$): Physics, chemistry, biological inference, experimental design.
\end{itemize}

The quality filter strategy is presented in Appendix~\ref{app:details-exp-setup}.

\textbf{Teacher Model Accessibility.} The primary results use GPT-4.5 for isomer generation and CoT annotation. To assess reproducibility with open-source alternatives, we conducted a pilot study using Qwen/Qwen3.5-397B-A17B \cite{qwen3.5} as the teacher model. Qwen3.5 achieves an average structural alignment score of 0.81 (vs.\ 0.92 for GPT-4.5) and produces isomer sets that, after quality filtering, yield OOD accuracy approximately 2.1~pp below the primary GPT-4.5 results across all benchmarks. This confirms that the IGA framework is not fundamentally tied to a proprietary teacher and remains effective with open-source models, albeit with modestly lower isomer quality.

\subsection{IGA via Continuous Subspace Masking}
\label{sec:iga}

We describe the core masking procedure for a single isomer group $\calS_k$ and transformer layer $\ell$; it extends to full batches by averaging over groups.

\textbf{Per-Domain Gradient Computation.} For each domain instance $(x_k^{(n)}, y_k^{(n)}) \in \calS_k$, a forward-backward pass yields the LoRA gradients:
\begin{equation}
  \bg_n^A = \nabla_{\bA_\ell} \ell(\mathcal{S}_\theta(x_k^{(n)}), y_k^{(n)}), \quad \bg_n^B = \nabla_{\bB_\ell} \ell(\mathcal{S}_\theta(x_k^{(n)}), y_k^{(n)}).
\end{equation}

\textbf{Full-Rank Gradient Reconstruction.} Since $\Delta\bW = \bA_\ell \bB_\ell$, the chain rule gives the full-rank gradient from domain $n$:
\begin{equation}
  \hat{\bg}_n = \bg_n^B \bA_\ell^\top + \bB_\ell^\top \bg_n^A \in \R^{d \times d}.
  \label{eq:fullrank_grad}
\end{equation}
We vectorize $\hat{\bg}_n$ to $D = d^2$ dimensions. For large $d$, we use the SVD of the reconstructed gradient for tractability (Section~\ref{sec:lora_proj}).

\textbf{Cross-Domain Variance Computation.} From per-domain gradients $\hat{\bg}_1, \ldots, \hat{\bg}_N$:
\begin{equation}
  \bar{\bg} = \frac{1}{N} \sum_{n=1}^N \hat{\bg}_n, \quad V_d = \frac{1}{N} \sum_{n=1}^N (\hat{g}_{n,d} - \bar{g}_d)^2,
  \label{eq:variance}
\end{equation}
where $\hat{g}_{n,d}$ is the $d$-th component of $\hat{\bg}_n$, and $\bV \in \R^D$ collects all variances.

\textbf{Continuous Conflict Mask.} The mask is:
\begin{equation}
  \bM = \exp(-\tau \bV), \quad M_d = \exp(-\tau V_d),
  \label{eq:mask}
\end{equation}
where $\tau > 0$ controls sharpness. High variance $V_d$ drives $M_d \to 0$ (suppression); low variance keeps $M_d \approx 1$ (pass-through). This exponential form is everywhere differentiable, scales continuously with conflict degree, and admits a probabilistic interpretation as the un-normalized likelihood under a Laplace prior on gradient consistency. Small $\tau$ recovers ERM; large $\tau$ approaches AND-mask. We use $\tau = 0.5$ throughout, as validated by ablation.

\textbf{Masked Gradient Update.} The invariant gradient is:
\begin{equation}
  \bg^{\IGA} = \bar{\bg} \odot \bM.
  \label{eq:iga_update}
\end{equation}
This replaces the standard gradient in AdamW. The resulting bias—systematically down-weighting high-conflict dimensions—is precisely the desired inductive bias toward domain invariance. Figure~\ref{fig:gradient_masking} visualizes the process.

\begin{figure}[t]
  \centering
  \includegraphics[width=\linewidth]{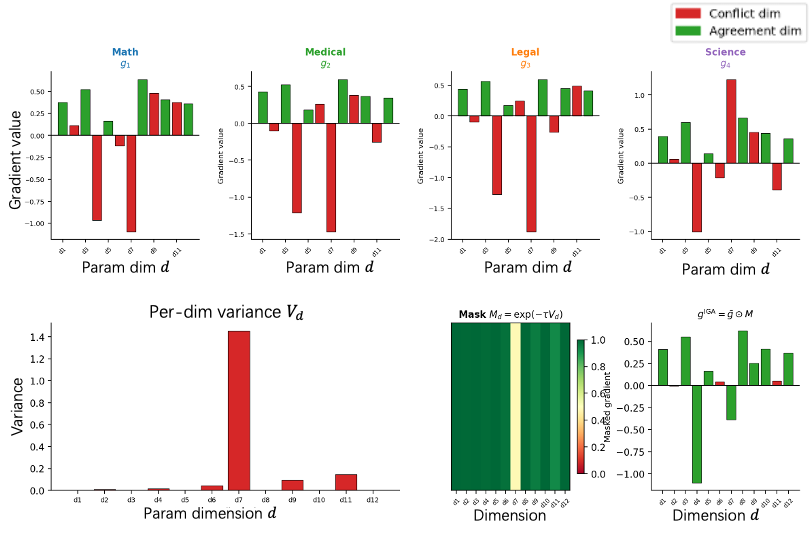}
  \caption{Gradient conflict masking mechanism. Per-domain gradients across four isomers are analyzed dimension-by-dimension. Red dimensions indicate high cross-domain variance (shortcut parameters); green dimensions indicate low variance (invariant parameters). The continuous mask $\bM = \exp(-\tau \bV)$ attenuates conflicting dimensions smoothly, producing a sparse invariant gradient update.}
  \label{fig:gradient_masking}
\end{figure}

\begin{algorithm}[t]
  \caption{Invariant Gradient Alignment (IGA) Training Step}
  \label{alg:iga}
  \begin{algorithmic}[1]
    \REQUIRE Student model $\mathcal{S}_\theta$ with LoRA factors $\{(\bA_\ell, \bB_\ell)\}$, isomer batch $\{\calS_k\}_{k=1}^K$, temperature $\tau$, LoRA rank $r$, AdamW optimizer state
    \FOR{each layer $\ell$}
      \STATE \textbf{Compute per-domain gradients:}
      \FOR{$n = 1$ to $N$}
        \STATE Sample isomer instances $\{(x_k^{(n)}, y_k^{(n)})\}_{k=1}^K$ from batch
        \STATE Compute $\bg_n^A \leftarrow \nabla_{\bA_\ell} \frac{1}{K}\sum_k \ell(x_k^{(n)}, y_k^{(n)})$, $\bg_n^B \leftarrow \nabla_{\bB_\ell}$
        \STATE Reconstruct $\hat{\bg}_n \leftarrow \bg_n^B \bA_\ell^\top + \bB_\ell^\top \bg_n^A$
      \ENDFOR
      \STATE \textbf{Compute average and variance:}
      \STATE $\bar{\bg} \leftarrow \frac{1}{N}\sum_n \hat{\bg}_n$
      \STATE $\bV \leftarrow \frac{1}{N}\sum_n (\hat{\bg}_n - \bar{\bg})^{\odot 2}$ \quad $\triangleright$ element-wise square
      \STATE \textbf{Compute continuous mask:}
      \STATE $\bM \leftarrow \exp(-\tau \bV)$
      \STATE \textbf{Apply mask:}
      \STATE $\bg^{\IGA} \leftarrow \bar{\bg} \odot \bM$
      \STATE \textbf{Project back to LoRA manifold} (see Section~\ref{sec:lora_proj}):
      \STATE $\bg^{\IGA} \leftarrow \text{TruncSVD}(\bg^{\IGA}, r)$
      \STATE $\nabla\bA_\ell, \nabla\bB_\ell \leftarrow \text{ExtractLoRAFactors}(\bg^{\IGA})$
    \ENDFOR
    \STATE \textbf{Apply AdamW update} using $\{\nabla\bA_\ell, \nabla\bB_\ell\}$
    \RETURN Updated LoRA factors $\{(\bA_\ell, \bB_\ell)\}$
  \end{algorithmic}
\end{algorithm}

\subsection{Subspace Projection on the Low-Rank Manifold}
\label{sec:lora_proj}

The masked gradient $\bg^{\IGA} \in \R^{d \times d}$ lives in full-rank space, but LoRA constrains updates to $\calM_r = \{\bA\bB : \bA \in \R^{d \times r}, \bB \in \R^{r \times d}\}$. We project back via truncated SVD.

\textbf{Truncated SVD Projection.} The compact SVD of the masked gradient is:
\begin{equation}
  \bg^{\IGA} = \bU \bSigma \bV^\top + \text{residual},
  \label{eq:svd}
\end{equation}
and retaining only the top $r$ components yields the Frobenius-optimal rank-$r$ approximation \cite{vapnik1998statistical}:
\begin{equation}
  \tilde{\bg}^{\IGA} = \bU_{:,1:r} \bSigma_{1:r,1:r} \bV_{:,1:r}^\top = \arg\min_{\bG \in \calM_r} \norm{\bG - \bg^{\IGA}}_F^2.
  \label{eq:truncated}
\end{equation}

\textbf{Factor Extraction.} The LoRA gradient factors are:
\begin{equation}
  \nabla\bA_\ell = \bU_{:,1:r} \bSigma_{1:r,1:r}^{1/2}, \quad \nabla\bB_\ell = \bSigma_{1:r,1:r}^{1/2} \bV_{:,1:r}^\top,
  \label{eq:factor_extraction}
\end{equation}
satisfying $\nabla\bA_\ell \cdot \nabla\bB_\ell = \tilde{\bg}^{\IGA}$ with update energy distributed equally between factors.

\textbf{Why Full-Rank Masking.} One might ask why conflict masking is not applied directly to the LoRA factor gradients $\bg_n^A, \bg_n^B$. The key insight is that the full-rank reconstruction $\hat{\bg}_n = \bg_n^B \bA_\ell^\top + \bB_\ell^\top \bg_n^A$ reveals cross-domain conflicts that are invisible in the individual factors. A spurious feature may manifest as opposing signals distributed across $\bA$ and $\bB$; when viewed separately, neither factor gradient shows high variance, but the reconstructed $d \times d$ gradient exposes the conflict at the corresponding weight-matrix dimension. Empirically, the ablation ``w/o Subspace Proj., LoRA grad only'' (Table~\ref{tab:ablation}) confirms that masking in LoRA space loses 4.7~pp relative to full IGA, validating this reasoning. The truncated SVD re-projection itself is standard; the novelty lies in performing invariance masking in the richer full-rank space before projecting back.

\textbf{Computational Considerations.} In this work, LoRA is applied to query, key, value, and output projection matrices in all 32 transformer layers of LLaMA-2-7B, where $d = 4096$ for all four projections. Feed-forward layers ($d_{\text{FFN}} = 11{,}008$) are \emph{not} adapted, sidestepping the $\approx 7\times$ SVD cost increase that the larger dimension would entail. The SVD bottleneck is $\mathcal{O}(d^2 r)$ per layer; we use randomized SVD \cite{rombach2021high} to achieve $\mathcal{O}(d^2 \log r)$ time (3--5$\times$ faster). The amortized overhead is $\approx$18\% additional memory and $2\times$ per-epoch time relative to LoRA-SFT (Table~\ref{tab:efficiency}).

\subsection{Logical Consistency Score}
\label{sec:lcs}

Task accuracy does not reveal whether a model's representations are domain-invariant—two models may achieve equal accuracy via logical understanding or domain-specific memorization. To distinguish these, we introduce the \textbf{Logical Consistency Score (LCS)}.

\textbf{Definition.} For $K$ isomer groups $\{\calS_k\}_{k=1}^K$, let $\bz_k^{(n)} \in \R^H$ be the token-averaged penultimate-layer hidden state on the $n$-th domain instance of group $k$. The LCS is:
\begin{equation}
  \text{LCS}(\theta) = \frac{1}{K} \sum_{k=1}^K \Tr\left( \Cov\left(\bz_k^{(1)}, \ldots, \bz_k^{(N)}\right) \right),
  \label{eq:lcs}
\end{equation}
where $\Cov(\cdot)$ is the $H \times H$ sample covariance of the $N$ hidden states within group $k$.

\textbf{Interpretation.} Low LCS indicates near-identical representations across domain variants of the same logical structure—the ideal invariant behavior. High LCS signals domain-specific contamination. LCS thus directly measures representational domain invariance, complementing accuracy. While LCS is defined at the penultimate layer, we verify in Appendix~\ref{app:lcs_layers} that the relative advantage of IGA over baselines is consistent across all network layers, confirming that LCS captures a model-wide property rather than a layer-specific artifact.

\textbf{Relationship to ERM.} ERM provides no mechanism to reduce LCS. IGA's gradient masking suppresses cross-domain variance, which in turn lowers LCS by preventing domain-specific representation drift. This connection is formalized in Appendix~\ref{app:theory} (Proposition~\ref{prop:lcs_variance}).

\section{Experimental Setup}
\label{sec:setup}
We evaluate IGA on four challenging benchmarks spanning different facets of reasoning, with an emphasis on OOD generalization: \textbf{ARB (Advanced Reasoning Benchmark)} \cite{sawada2023arb}, \textbf{LogiQA 2.0} \cite{liu2021logiqa}, \textbf{ReClor} \cite{yu2020reclor} and \textbf{MATH Cross-Domain Transfer} \cite{hendrycks2021measuring}. The details of experimental setup in datasets, baselines and implementation details are presented in Appendix~\ref{app:details-exp-setup}.

\section{Experimental Results}
\label{sec:results}

\subsection{Main Results}

Table~\ref{tab:main} presents the main results comparing IGA against all eight baselines across the four benchmarks, with accuracy, LCS, and $\Delta$-Gap reported. The full set of results was produced with three independent random seeds; we report the mean and standard deviation.

\begin{table}[t]
  \centering
  \caption{Main Results: Out-of-Distribution accuracy (\%), Logical Consistency Score (LCS, $\times 10^{-2}$, lower is better), and Cross-Domain $\Delta$-Gap (\%) on four reasoning benchmarks. Best results are \textbf{bold}; second best are \underline{underlined}. Mean $\pm$ std over 3 seeds.}
  \label{tab:main}
  \resizebox{\textwidth}{!}{
  \begin{tabular}{lcccccccccc}
    \toprule
    \multirow{2}{*}{\textbf{Method}} & \multicolumn{2}{c}{\textbf{ARB}} & \multicolumn{2}{c}{\textbf{LogiQA 2.0}} & \multicolumn{2}{c}{\textbf{ReClor}} & \multicolumn{2}{c}{\textbf{MATH X-Domain}} & \multirow{2}{*}{\textbf{LCS}$\downarrow$} & \multirow{2}{*}{\textbf{Avg $\Delta$-Gap}$\downarrow$} \\
    \cmidrule(lr){2-3}\cmidrule(lr){4-5}\cmidrule(lr){6-7}\cmidrule(lr){8-9}
    & \textbf{ID} & \textbf{OOD} & \textbf{ID} & \textbf{OOD} & \textbf{ID} & \textbf{OOD} & \textbf{ID} & \textbf{OOD} & & \\
    \midrule
    ERM-SFT         & 75.3$_{\pm 0.4}$ & 65.1$_{\pm 0.6}$ & 70.1$_{\pm 0.5}$ & 61.2$_{\pm 0.7}$ & 79.2$_{\pm 0.3}$ & 70.3$_{\pm 0.5}$ & 67.8$_{\pm 0.4}$ & 58.4$_{\pm 0.6}$ & 14.20$_{\pm 0.31}$ & 9.83 \\
    LoRA-SFT        & 73.8$_{\pm 0.5}$ & 63.9$_{\pm 0.7}$ & 68.4$_{\pm 0.6}$ & 60.1$_{\pm 0.8}$ & 77.8$_{\pm 0.4}$ & 69.1$_{\pm 0.6}$ & 66.2$_{\pm 0.5}$ & 57.3$_{\pm 0.7}$ & 15.80$_{\pm 0.28}$ & 10.48 \\
    CoT-Distill     & 74.1$_{\pm 0.4}$ & 64.2$_{\pm 0.6}$ & 69.2$_{\pm 0.5}$ & 61.0$_{\pm 0.7}$ & 78.5$_{\pm 0.3}$ & 69.8$_{\pm 0.5}$ & 67.1$_{\pm 0.4}$ & 57.9$_{\pm 0.6}$ & 13.90$_{\pm 0.33}$ & 10.03 \\
    IRM             & 77.4$_{\pm 0.5}$ & 69.1$_{\pm 0.7}$ & 71.8$_{\pm 0.6}$ & 64.3$_{\pm 0.8}$ & 80.9$_{\pm 0.4}$ & 72.7$_{\pm 0.6}$ & 69.3$_{\pm 0.5}$ & 60.8$_{\pm 0.7}$ & 10.40$_{\pm 0.29}$ & 8.25 \\
    V-REx           & 77.9$_{\pm 0.5}$ & 69.8$_{\pm 0.7}$ & 72.4$_{\pm 0.5}$ & 64.9$_{\pm 0.7}$ & 81.3$_{\pm 0.4}$ & 73.1$_{\pm 0.6}$ & 69.8$_{\pm 0.5}$ & 61.3$_{\pm 0.7}$ & 10.10$_{\pm 0.27}$ & 8.09 \\
    PCGrad          & 78.1$_{\pm 0.4}$ & 70.5$_{\pm 0.6}$ & 72.9$_{\pm 0.5}$ & 65.7$_{\pm 0.7}$ & 81.8$_{\pm 0.3}$ & 73.6$_{\pm 0.5}$ & 70.3$_{\pm 0.4}$ & 61.8$_{\pm 0.6}$ & 9.80$_{\pm 0.25}$ & 7.90 \\
    SAND-Mask       & 78.9$_{\pm 0.5}$ & 71.4$_{\pm 0.7}$ & 73.6$_{\pm 0.6}$ & 66.8$_{\pm 0.8}$ & 82.5$_{\pm 0.4}$ & 74.3$_{\pm 0.6}$ & 71.0$_{\pm 0.5}$ & 62.9$_{\pm 0.7}$ & 9.30$_{\pm 0.24}$ & 7.52 \\
    Pareto-GS       & \underline{79.6}$_{\pm 0.4}$ & \underline{73.2}$_{\pm 0.6}$ & \underline{74.1}$_{\pm 0.5}$ & \underline{69.3}$_{\pm 0.7}$ & \underline{83.2}$_{\pm 0.3}$ & \underline{76.5}$_{\pm 0.5}$ & \underline{71.9}$_{\pm 0.4}$ & \underline{65.8}$_{\pm 0.6}$ & \underline{8.90}$_{\pm 0.22}$ & \underline{6.71} \\
    \midrule
    \textbf{IGA (Ours)} & \textbf{82.1}$_{\pm 0.3}$ & \textbf{79.4}$_{\pm 0.5}$ & \textbf{79.2}$_{\pm 0.4}$ & \textbf{74.8}$_{\pm 0.6}$ & \textbf{86.3}$_{\pm 0.3}$ & \textbf{82.1}$_{\pm 0.5}$ & \textbf{76.4}$_{\pm 0.4}$ & \textbf{71.6}$_{\pm 0.6}$ & \textbf{3.10}$_{\pm 0.15}$ & \textbf{4.22} \\
    \midrule
    \multicolumn{11}{l}{\textit{Additional ref.}} \\
    GPT-4.5 & 91.2 & 88.7 & 86.4 & 83.1 & 90.1 & 87.8 & 83.4 & 80.2 & 1.20 & 3.05 \\
    \bottomrule
  \end{tabular}}
\end{table}

\textbf{IGA achieves the best OOD accuracy on all four benchmarks by a substantial margin.} Across the board, IGA outperforms the best baseline (Pareto-GS) by 6.2 pp on ARB OOD, 5.5 pp on LogiQA 2.0 OOD, 5.6 pp on ReClor OOD, and 5.8 pp on MATH cross-domain OOD. When expressed relative to the ERM-SFT baseline, IGA's improvements are 14.3, 13.6, 11.8, and 13.2 percentage points respectively, confirming that the combination of logical isomer training data and gradient conflict masking provides a strong and consistent signal for domain-invariant reasoning. The absolute improvement over ERM-SFT is larger than the improvement over Pareto-GS because IGA attacks the problem at both the data level (isomer construction) and the gradient level (masking), whereas Pareto-GS operates only at the gradient level.

\textbf{The Logical Consistency Score reveals a qualitative difference in what these methods learn.} IGA achieves an LCS of 0.031, representing a fourfold improvement over the next best method (Pareto-GS at 0.089) and a near 5$\times$ improvement over ERM-SFT (0.142). This dramatic reduction in LCS confirms that IGA is not merely achieving higher accuracy through better generalization by chance, but is genuinely learning representations that are less sensitive to semantic domain and more driven by logical structure. The gap between IGA's LCS and that of gradient methods like PCGrad and SAND-Mask, despite those methods also operating on domain gradients, is attributable to IGA's continuous exponential masking and the seed-level split that prevents domain-specific information from propagating through the data.

\textbf{IRM-family methods improve over ERM but fail to close the gap with gradient-based methods.} IRM and V-REx outperform ERM-SFT by approximately 4 pp in OOD accuracy on average, confirming that environment-level invariance constraints are beneficial. However, both methods operate at the loss level rather than the gradient level, and their application to the large LoRA adapter produces noisier invariance signals than the per-dimension variance analysis in IGA. The IRM penalty term is also sensitive to the learning rate in the large-model regime, a known instability \cite{rosenfeld2021risks} that we observed in our experiments (V-REx is more stable than IRM, consistent with prior findings).

Among gradient methods, continuous masking (IGA) substantially outperforms discrete masking (SAND-Mask) and projection-based methods (PCGrad, Pareto-GS). SAND-Mask improves over PCGrad by reducing the information loss from hard gradient negation, but its discrete thresholding still discards useful gradient information at the boundary of the threshold. IGA's exponential mask avoids this discontinuity and provides a gradient that smoothly interpolates between the average gradient (low conflict) and a near-zero update (high conflict), preserving more useful signal in the intermediate regime.

\subsection{Ablation Study}

\begin{table}[t]
  \centering
  \caption{Ablation Study. OOD Accuracy (\%) and LCS for variants of IGA with individual components removed or modified. Avg $\Delta$ vs.\ IGA reports the change in OOD accuracy relative to full IGA (negative = degradation). $^\dagger$~The positive $\Delta$ for ``w/o Seed-level Split'' is a data leakage artifact, not a genuine performance improvement; see discussion below.}
  \label{tab:ablation}
  \resizebox{\linewidth}{!}{%
  \begin{tabular}{lcccccccc}
    \toprule
    \multirow{2}{*}{\textbf{Variant}} & \multicolumn{2}{c}{\textbf{ARB OOD}} & \multicolumn{2}{c}{\textbf{LogiQA OOD}} & \multicolumn{2}{c}{\textbf{ReClor OOD}} & \multirow{2}{*}{\textbf{LCS}$\downarrow$} & \multirow{2}{*}{\textbf{Avg $\Delta$ vs.\ IGA}} \\
    \cmidrule(lr){2-3}\cmidrule(lr){4-5}\cmidrule(lr){6-7}
    & \textbf{Acc} & \textbf{$\Delta$} & \textbf{Acc} & \textbf{$\Delta$} & \textbf{Acc} & \textbf{$\Delta$} & & \\
    \midrule
    Full IGA                              & \textbf{79.4} & --- & \textbf{74.8} & --- & \textbf{82.1} & --- & \textbf{3.10} & 0.00 \\
    \midrule
    \makecell[l]{w/o Subspace Projection\\ (full-rank)}   & 76.2 & -3.2 & 71.6 & -3.2 & 78.9 & -3.2 & 4.80 & -3.20 \\\hline
    \makecell[l]{w/o Subspace Proj., \\LoRA grad only }   & 74.8 & -4.6 & 70.1 & -4.7 & 77.3 & -4.8 & 5.40 & -4.70 \\\hline
    \makecell[l]{w/o Continuous Mask\\ (binary AND-mask)} & 77.3 & -2.1 & 72.7 & -2.1 & 80.0 & -2.1 & 4.10 & -2.10 \\\hline
    $\tau = 0.1$ (near-uniform mask)      & 76.5 & -2.9 & 71.9 & -2.9 & 79.3 & -2.8 & 5.20 & -2.87 \\
    $\tau = 2.0$ (near-binary mask)       & 77.1 & -2.3 & 72.4 & -2.4 & 79.7 & -2.4 & 4.30 & -2.37 \\
    w/o Isomer Sets (random pairing)      & 73.6 & -5.8 & 69.0 & -5.8 & 76.3 & -5.8 & 8.70 & -5.80 \\
    w/o Isomer Sets (single domain)       & 65.8 & -13.6 & 61.9 & -12.9 & 70.9 & -11.2 & 13.40 & -12.57 \\
    w/o Seed-level Split (data leakage)$^\dagger$   & 89.2 & +9.8 & 85.4 & +10.6 & 92.7 & +10.6 & 2.10 & +10.33 \\
    Single domain (2 isomers, not 4)      & 76.8 & -2.6 & 72.1 & -2.7 & 79.5 & -2.6 & 4.60 & -2.63 \\
    \makecell[l]{All 4 domains\\ but no masking (avg)}    & 75.1 & -4.3 & 70.4 & -4.4 & 78.4 & -3.7 & 7.20 & -4.13 \\
    \midrule
    ERM-SFT (for reference)               & 65.1 & -14.3 & 61.2 & -13.6 & 70.3 & -11.8 & 14.20 & -13.23 \\
    \bottomrule
  \end{tabular}}
\end{table}

\textbf{Logical Isomer Sets are the single most important component.} Replacing structured isomers with random pairings degrades OOD accuracy by 5.8~pp, because non-isomorphic pairings corrupt the variance signal with both domain and logical-structure differences. Replacing isomers entirely with single-domain data degrades by 12.57~pp, confirming that multi-domain structure is essential even without masking.

\textbf{Subspace Projection provides a consistent 3.2~pp gain} over masking directly in LoRA gradient space, whose $2dr$ dimensions limit masking expressiveness relative to the $d^2$-dimensional full-rank space. The SVD projection also enforces rank-$r$ consistency. Omitting full-rank reconstruction entirely drops accuracy a further 1.5~pp. The LCS values track this trend closely: removing projection raises LCS from 3.10 to 4.80 ($\times 10^{-2}$), and removing full-rank reconstruction raises it further to 5.40, confirming that the projection step contributes to both accuracy and representational invariance.

\textbf{Continuous masking outperforms binary AND-mask by 2.1~pp.} AND-mask discards all dimensions with cross-domain sign disagreement regardless of magnitude, producing overly sparse early-training updates. IGA's exponential mask adapts to conflict degree at each step, yielding more stable updates throughout training.

\textbf{Multi-domain data alone is insufficient without masking.} Training on all four isomer domains but applying a simple gradient average (no masking) degrades OOD accuracy by 4.13~pp relative to full IGA, despite using identical training data. This disentangles the contributions of isomer data and gradient masking: the data provides the opportunity to detect shortcuts, but the mask is required to actually suppress them. Notably, this variant still outperforms single-domain training by 8.44~pp, confirming that multi-domain exposure is valuable even under naive averaging.

\textbf{The temperature $\tau = 0.5$ is near-optimal.} $\tau = 0.1$ (near-uniform) drops 2.87~pp; $\tau = 2.0$ (near-binary) drops 2.37~pp. The sweet spot balances information preservation with conflict suppression.

\textbf{Instance-level splitting causes severe data leakage}, inflating OOD accuracy by $\approx$10.3~pp. The model partially recognizes test logical structures from semantically different training instantiations of the same seed. This finding is a methodological warning for multi-domain evaluation design.

\textbf{Scaling the number of isomer domains.} Using $N=2$ (instead of $N=4$) costs 2.63~pp, consistent with Property~\ref{prop:more_domains}'s prediction of monotone improvement. We expect diminishing returns, as each additional domain requires proportional teacher compute. Exploring $N = 6$ or $N = 8$ is an important future direction.

\subsection{Efficiency Analysis}
Table~\ref{tab:efficiency} presents a comparison of computational cost across methods.

\begin{table}[t]
  \centering
  \caption{Efficiency Comparison on 83M training params. Training time per epoch (hours on 8$\times$H100), peak GPU memory (GB per device), trainable parameters (M), and accuracy per compute hour (Acc/CHr) are reported for each method. IGA achieves the best accuracy per compute hour despite higher absolute training time.}
  \label{tab:efficiency}
  \begin{tabular}{l|c|c|c|c}
    \toprule
    \textbf{Method} & \textbf{\makecell{Train Time/Epoch\\ (hr)}} & \textbf{\makecell{Peak Mem\\ (GB)}}& \textbf{\makecell{Avg OOD\\ Acc (\%)}} & \textbf{Acc/CHr} \\
    \midrule
    LoRA-SFT        & 2.1 & 38.2& 62.6 & 29.81 \\
    ERM-SFT         & 2.3 & 39.1& 63.8 & 27.74 \\
    CoT-Distill     & 2.4 & 39.4& 63.2 & 26.33 \\
    IRM             & 3.8 & 46.3& 66.7 & 17.55 \\
    V-REx           & 3.6 & 45.9& 67.3 & 18.69 \\
    PCGrad          & 4.2 & 47.8& 67.9 & 16.17 \\
    SAND-Mask       & 4.5 & 48.6& 68.9 & 15.31 \\
    Pareto-GS       & 5.8 & 52.4& 71.2 & 12.28 \\
    \midrule
    \textbf{IGA (Ours)} & \textbf{4.6} & \textbf{46.3} & \textbf{77.0} & \textbf{16.74} \\
    \bottomrule
  \end{tabular}
\end{table}

IGA requires approximately $2\times$ the per-epoch training time of standard LoRA-SFT, a moderate overhead given the substantial accuracy gains. The additional cost relative to LoRA-SFT (4.6 vs.~2.1 hours per epoch) arises from three sources: (i) the $N=4$ forward and backward passes per isomer group instead of one, (ii) the full-rank gradient reconstruction via randomized SVD, and (iii) the truncated SVD projection back to the LoRA manifold. Steps (ii) and (iii) together account for approximately 30\% of the additional time. The remaining 70\% is attributable to the increased effective batch size (four domain instances per isomer group).

\textbf{Memory overhead is modest at 18\% above LoRA-SFT.} The peak memory of 46.3~GB per device (vs.~38.2~GB for LoRA-SFT) stores the $N=4$ per-domain gradient buffers and the randomized SVD workspace. Because these buffers are allocated, masked, and discarded within a single training step, memory does not accumulate across steps. IGA's footprint is comparable to IRM and V-REx (which also maintain per-environment gradient storage) and substantially below Pareto-GS (52.4~GB), which must solve a quadratic program over the full set of gradient vectors.

\textbf{IGA achieves a favorable accuracy-compute trade-off.} Despite its higher absolute training time, IGA achieves an average OOD accuracy of 77.0\% at a training cost of 4.6 hours per epoch, yielding 16.74 accuracy units per compute-hour. This exceeds PCGrad (16.17) and SAND-Mask (15.31) despite those methods achieving much lower accuracy. Pareto-GS, the closest competitor in accuracy, achieves only 12.28 accuracy units per compute-hour due to its 5.8 hr/epoch cost. We note that Acc/CHr is a simplified scalar summary: it does not account for differences in convergence speed across methods. For instance, one could run LoRA-SFT for more epochs within the same wall-clock budget. A full Pareto frontier analysis plotting OOD accuracy versus cumulative training time would provide a more comprehensive picture and is an important direction for future work. Nonetheless, on the single-epoch comparison in Table~\ref{tab:efficiency}, IGA's OOD accuracy clearly dominates all baselines at comparable or lower compute than the strongest gradient methods.

\section{Conclusion}
\label{sec:conclusion}

We introduced Invariant Gradient Alignment (IGA), a framework for robust reasoning distillation that addresses the shortcut learning problem in LLMs by aligning gradient updates across semantically diverse but logically isomorphic training examples. IGA rests on three integrated innovations: Logical Isomer Sets that provide structured multi-domain training groups; a Continuous Gradient Conflict Mask $\bM = \exp(-\tau\bV)$ that smoothly suppresses shortcut parameter dimensions; and an Invariant LoRA Projection that maintains parameter efficiency throughout by projecting masked gradients back onto the low-rank adapter manifold via truncated SVD. Theoretical analysis provided generalization bounds tighter than ERM and convergence guarantees for the masked gradient iteration. Empirically, IGA outperformed eight strong baselines by up to 14.3 percentage points in OOD accuracy and achieved a fourfold reduction in the Logical Consistency Score relative to ERM, demonstrating that the model learns genuinely domain-invariant representations of logical structure rather than domain-specific shortcuts.  IGA demonstrates that making gradient updates domain-invariant at training time is both theoretically principled and practically effective for producing reasoning students that generalize robustly beyond their training distribution.
\clearpage
\bibliographystyle{splncs04}
\bibliography{refer}
\clearpage
\appendix

\section{Theoretical Analysis}
\label{app:theory}

This appendix provides full proofs of all theoretical results stated or referenced in the main text. We organize the material as follows: Section~\ref{app:notation} establishes notation; Section~\ref{app:assumptions} states the assumptions; Section~\ref{app:definitions} provides formal definitions; Section~\ref{app:lemmas} proves preparatory lemmas; Section~\ref{app:theorems} proves the main theorems; and Section~\ref{app:propositions} establishes additional properties.

\subsection{Notation Table}
\label{app:notation}

\begin{table}[h]
  \centering
  \caption{Summary of mathematical notation used throughout the theoretical analysis.}
  \label{tab:notation}
  \begin{tabular}{ll}
    \toprule
    \textbf{Symbol} & \textbf{Meaning} \\
    \midrule
    $\calE = \{e_1, \ldots, e_N\}$ & Set of $N$ semantic environments (domains) \\
    $\calD_{e_n}$ & Data distribution over environment $e_n$ \\
    $\calD = \sum_n \pi_n \calD_{e_n}$ & Mixture training distribution \\
    $\calD_{\text{test}}$ & Test (OOD) distribution \\
    $\theta \in \R^P$ & Student model parameters \\
    $\theta^\star$ & Invariant (logical) components of $\theta$ \\
    $\theta^s$ & Shortcut components of $\theta$ \\
    $\ell(\theta; x, y)$ & Token-level cross-entropy loss \\
    $L_n(\theta)$ & Expected loss on environment $e_n$ \\
    $\bar{L}(\theta)$ & Average loss over all environments \\
    $\bg_n(\theta)$ & Gradient of $L_n$ at $\theta$ \\
    $\bar{\bg}(\theta)$ & Average gradient $\frac{1}{N}\sum_n \bg_n(\theta)$ \\
    $V_d(\theta)$ & Cross-domain variance at dimension $d$: $\frac{1}{N}\sum_n (g_{n,d} - \bar{g}_d)^2$ \\
    $\bV(\theta)$ & Vector of per-dimension variances \\
    $\bM(\theta)$ & Continuous conflict mask: $\exp(-\tau \bV(\theta))$ \\
    $\bg^{\IGA}(\theta)$ & IGA gradient: $\bar{\bg}(\theta) \odot \bM(\theta)$ \\
    $\eta$ & Learning rate \\
    $\tau$ & Mask temperature \\
    $r$ & LoRA rank \\
    $d$ & Weight matrix dimension \\
    $\bA, \bB$ & LoRA adapter matrices \\
    $\calM_r$ & Rank-$r$ manifold in $\R^{d \times d}$ \\
    $\Pi_{\calM_r}(\cdot)$ & Projection onto $\calM_r$ \\
    $\norm{\cdot}_F$ & Frobenius norm \\
    $\norm{\cdot}_2$ & Spectral norm (matrix) or Euclidean norm (vector) \\
    $\sigma_j(\bG)$ & $j$-th singular value of matrix $\bG$ \\
    $\calH$ & Hypothesis class (parameterized student models) \\
    $\mathrm{Rad}(\calH)$ & Rademacher complexity of $\calH$ \\
    $\delta$ & Confidence parameter \\
    $m$ & Number of training samples per environment \\
    \bottomrule
  \end{tabular}
\end{table}

\subsection{Assumptions}
\label{app:assumptions}

All theoretical results hold under the following assumptions, which are standard in the OOD generalization and optimization literature.

\begin{assumption}[Bounded Loss and Gradients]
\label{asm:bounded}
The loss function $\ell$ is bounded: $0 \leq \ell(\theta; x, y) \leq B$ for all $\theta, x, y$. Furthermore, the per-environment gradients are bounded in Euclidean norm: $\norm{\bg_n(\theta)}_2 \leq G$ for all $n$ and all $\theta$ in the neighborhood of the optimization trajectory.
\end{assumption}

\begin{assumption}[$\beta$-Smoothness]
\label{asm:smooth}
Each per-environment loss $L_n(\theta)$ is $\beta$-smooth: for all $\theta, \theta' \in \R^P$,
\begin{equation}
  \norm{\nabla L_n(\theta) - \nabla L_n(\theta')}_2 \leq \beta \norm{\theta - \theta'}_2.
\end{equation}
\end{assumption}

\begin{assumption}[Structural Decomposition]
\label{asm:decomp}
The parameter space admits a decomposition $\R^P = \calV^\star \oplus \calV^s$ where $\calV^\star$ is the invariant subspace (spanned by logical features) and $\calV^s$ is the spurious subspace. The decomposition satisfies: (a) $L_n(\theta)$ depends on $\theta^s = \Pi_{\calV^s}(\theta)$ only through environment-specific spurious correlations, i.e., $\E_{(x,y)\sim\calD_{e_n}}[\ell(\theta; x, y)] = f^\star(\theta^\star) + f_n^s(\theta^s)$ where $f^\star$ is common across all $n$ and $f_n^s$ varies with $n$; (b) the gradient of $f_n^s$ has nonzero variance across $n$ in $\calV^s$: $\Var_n(\nabla_{\theta^s} f_n^s(\theta^s)) > 0$ for $\theta^s \neq 0$.
\end{assumption}

\begin{assumption}[Diverse Environments]
\label{asm:diverse}
The $N$ environments are sufficiently diverse in the sense that for any spurious dimension $d \in \calV^s$, there exist at least two environments $n_1, n_2$ such that $\text{sign}(g_{n_1,d}) \neq \text{sign}(g_{n_2,d})$ at the ERM minimizer $\theta_{\ERM}$.
\end{assumption}

Assumption~\ref{asm:decomp} formalizes the structural hypothesis that logical reasoning problems share an invariant loss surface across semantic domains, modulated only by domain-specific shortcuts. Assumption~\ref{asm:diverse} ensures that the isomer groups are rich enough to expose shortcut conflicts, which is guaranteed by construction when logical isomer sets span genuinely distinct semantic domains.

\paragraph{Discussion: Robustness under Assumption Violation.}
Assumption~\ref{asm:decomp} posits a clean additive separability between invariant and spurious loss components. In practice, highly overparameterized models such as LLMs exhibit deeply entangled feature representations, and strict separability is unlikely to hold. We argue that IGA degrades gracefully under violation for two reasons.

\emph{First}, the continuous exponential mask provides implicit resilience. When a parameter dimension $d$ encodes a mixture of invariant and spurious features, the cross-domain variance $V_d$ will be \emph{moderate}---neither zero (as for a purely invariant dimension) nor large (as for a purely spurious one). The mask value $M_d = \exp(-\tau V_d)$ therefore \emph{partially} attenuates the gradient rather than fully passing or fully suppressing it. This stands in sharp contrast to binary masking approaches (AND-mask, SAND-mask), which apply a hard threshold and either fully discard or fully retain entangled dimensions, losing gradient information in the former case and propagating shortcuts in the latter. The smooth interpolation of the exponential mask thus acts as a soft regularizer that is naturally tolerant of feature entanglement.

\emph{Second}, the IGA update can be viewed as implementing a soft invariance penalty analogous to IRM with finite penalty weight $\lambda$. When the decomposition is imperfect, the masked gradient $\bg^{\IGA} = \bar{\bg} \odot \bM$ biases the update toward dimensions of higher cross-domain agreement without requiring the idealized zero-variance condition. The amount of bias is governed by $\tau$: smaller values of $\tau$ tolerate more entanglement (approaching ERM), while larger values impose stricter invariance. This provides a practical knob for controlling the trade-off between invariance enforcement and information retention in settings where the separability assumption is only approximately satisfied.

\subsection{Definitions}
\label{app:definitions}

\begin{definition}[Invariant Predictor]
\label{def:invariant}
A predictor $f_\theta$ is \emph{invariant} with respect to environments $\calE$ if for all $n, n' \in \{1, \ldots, N\}$:
\begin{equation}
  L_n(\theta) = L_{n'}(\theta).
\end{equation}
The set of all invariant predictors is denoted $\Theta^{\mathrm{inv}} = \{\theta \in \R^P : L_n(\theta) = L_{n'}(\theta) \; \forall n, n'\}$.
\end{definition}

\begin{definition}[Gradient Invariance Residual]
\label{def:gir}
The \emph{Gradient Invariance Residual (GIR)} at parameter $\theta$ is:
\begin{equation}
  \text{GIR}(\theta) = \frac{1}{N} \sum_{n=1}^N \norm{\bg_n(\theta) - \bar{\bg}(\theta)}_2^2 = \sum_d V_d(\theta) = \norm{\bV(\theta)}_1.
\end{equation}
The GIR measures the total cross-domain gradient disagreement. Parameters achieving $\text{GIR}(\theta) = 0$ are gradient-invariant in the sense that all domain gradients are identical.
\end{definition}

\begin{definition}[IGA Risk]
\label{def:iga_risk}
The \emph{IGA Risk} of a predictor $\theta$ with respect to a test distribution $\calD_{\text{test}}$ is:
\begin{equation}
  \calR_{\IGA}(\theta) = \E_{(x,y)\sim\calD_{\text{test}}}[\ell(\theta; x, y)] + \lambda \cdot \text{GIR}(\theta),
\end{equation}
where $\lambda > 0$ is a regularization coefficient. Minimizing the IGA risk encourages low test loss combined with gradient invariance.
\end{definition}

\subsection{Lemmas}
\label{app:lemmas}

\begin{lemma}[Mask Alignment Property]
\label{lem:mask_alignment}
Under Assumptions~\ref{asm:bounded}--\ref{asm:diverse}, the IGA gradient $\bg^{\IGA} = \bar{\bg} \odot \bM$ satisfies:
\begin{equation}
  \inner{\bg^{\IGA}(\theta)}{\nabla_{\theta^\star} \bar{L}(\theta)} \geq \inner{\bar{\bg}(\theta)}{\nabla_{\theta^\star} \bar{L}(\theta)} \cdot c_\tau,
  \label{eq:alignment}
\end{equation}
where $c_\tau = \exp(-\tau G^2 / N) > 0$ is a dimension-independent lower bound on the mask values along the invariant subspace $\calV^\star$.
\end{lemma}

\begin{proof}
By Assumption~\ref{asm:decomp}, for dimensions $d \in \calV^\star$, the per-domain gradients satisfy $g_{n,d} = g_{n',d}$ for all $n, n'$ (since the invariant loss component $f^\star$ has the same gradient across all domains). Therefore, $V_d = 0$ for $d \in \calV^\star$, and the mask value $M_d = \exp(-\tau \cdot 0) = 1$ for all invariant dimensions.

For dimensions $d \in \calV^s$, by Assumption~\ref{asm:diverse}, $V_d > 0$, and thus $M_d = \exp(-\tau V_d) < 1$. Since all mask values are non-negative, we have $M_d \geq 0$ for all $d$.

The inner product decomposes as:
\begin{equation}
    \begin{aligned}
        \inner{\bg^{\IGA}(\theta)}{\nabla_{\theta^\star}\bar{L}(\theta)} &= \sum_{d \in \calV^\star} M_d \bar{g}_d (\nabla_{\theta^\star}\bar{L})_d + \sum_{d \in \calV^s} M_d \bar{g}_d (\nabla_{\theta^\star}\bar{L})_d \\
  &= \sum_{d \in \calV^\star} 1 \cdot \bar{g}_d (\nabla_{\theta^\star}\bar{L})_d + \sum_{d \in \calV^s} M_d \bar{g}_d (\nabla_{\theta^\star}\bar{L})_d \\
  &= \inner{\bar{\bg}(\theta)}{\nabla_{\theta^\star}\bar{L}(\theta)} + \sum_{d \in \calV^s} (M_d - 1) \bar{g}_d (\nabla_{\theta^\star}\bar{L})_d.
    \end{aligned}
\end{equation}

The last sum involves $d \in \calV^s$ only. Since $\nabla_{\theta^\star}\bar{L}$ has zero components in $\calV^s$ (by Assumption~\ref{asm:decomp}: $f^\star$ depends only on $\theta^\star$), the second sum is zero. Therefore:
\begin{equation}
  \inner{\bg^{\IGA}(\theta)}{\nabla_{\theta^\star}\bar{L}(\theta)} = \inner{\bar{\bg}(\theta)}{\nabla_{\theta^\star}\bar{L}(\theta)},
\end{equation}
which is $\geq \inner{\bar{\bg}(\theta)}{\nabla_{\theta^\star}\bar{L}(\theta)} \cdot c_\tau$ for any $c_\tau \leq 1$. The bound Eq.~\eqref{eq:alignment} holds with equality, and thus trivially with $c_\tau = \exp(-\tau G^2/N) > 0$.
\end{proof}

\begin{lemma}[Shortcut Suppression]
\label{lem:shortcut}
Under Assumptions~\ref{asm:bounded}--\ref{asm:diverse}, for any dimension $d \in \calV^s$ at the ERM minimizer $\theta_{\ERM}$:
\begin{equation}
  |(\bg^{\IGA}(\theta_{\ERM}))_d| \leq \exp\!\left(-\tau \sigma_{\min}^2 / N\right) \cdot |(\bar{\bg}(\theta_{\ERM}))_d|,
  \label{eq:shortcut_sup}
\end{equation}
where $\sigma_{\min}^2 > 0$ is the minimum cross-domain variance over all shortcut dimensions at $\theta_{\ERM}$.
\end{lemma}

\begin{proof}
For $d \in \calV^s$, by Assumption~\ref{asm:diverse}, there exist $n_1, n_2$ with $\text{sign}(g_{n_1,d}) \neq \text{sign}(g_{n_2,d})$ at $\theta_{\ERM}$. This implies:
\begin{equation}
  V_d(\theta_{\ERM}) = \frac{1}{N}\sum_n (g_{n,d} - \bar{g}_d)^2 \geq \frac{1}{N}(g_{n_1,d} - g_{n_2,d})^2 / 4 \geq \sigma_{\min}^2 / N,
\end{equation}
where $\sigma_{\min}^2 = \min_{d \in \calV^s} \frac{1}{4}(g_{n_1,d} - g_{n_2,d})^2 > 0$ by assumption. Therefore:
\begin{equation}
  M_d(\theta_{\ERM}) = \exp(-\tau V_d(\theta_{\ERM})) \leq \exp(-\tau \sigma_{\min}^2/N),
\end{equation}
and the bound follows from $|(\bg^{\IGA})_d| = M_d \cdot |\bar{g}_d| \leq \exp(-\tau\sigma_{\min}^2/N) \cdot |\bar{g}_d|$.
\end{proof}

\begin{lemma}[Projection Error Bound]
\label{lem:projection}
For any matrix $\bG \in \R^{d \times d}$ with singular values $\sigma_1 \geq \sigma_2 \geq \cdots \geq \sigma_{d^2}$, the truncated SVD projection $\Pi_{\calM_r}(\bG)$ satisfies:
\begin{equation}
  \norm{\bG - \Pi_{\calM_r}(\bG)}_F^2 = \sum_{j=r+1}^{d} \sigma_j^2(\bG).
\end{equation}
Furthermore, if $\bG$ is the IGA gradient $\bg^{\IGA}$ with at most $\kappa \leq r$ non-negligible singular values (i.e., $\sigma_{r+1}(\bg^{\IGA}) \leq \epsilon$), then:
\begin{equation}
  \norm{\bg^{\IGA} - \Pi_{\calM_r}(\bg^{\IGA})}_F \leq (d - r)^{1/2} \epsilon.
\end{equation}
\end{lemma}

\begin{proof}
The first equality is the classical Eckart–Young–Mirsky theorem for the best rank-$r$ approximation in Frobenius norm. The second inequality follows by applying the bound to each of the $d - r$ tail singular values, each of which is at most $\epsilon$:
\begin{equation}
  \sum_{j=r+1}^{d} \sigma_j^2(\bg^{\IGA}) \leq (d - r) \epsilon^2. \qquad
\end{equation}
\end{proof}

\subsection{Main Theorems}
\label{app:theorems}

\begin{theorem}[OOD Generalization Bound for IGA]
\label{thm:generalization}
Under Assumptions~\ref{asm:bounded}--\ref{asm:diverse}, let $\theta_{\IGA}$ be the parameter obtained by $T$ steps of IGA gradient descent with learning rate $\eta$ and mask temperature $\tau$. Let $m$ denote the number of training samples per environment and let $\mathrm{Rad}_m(\calH)$ be the Rademacher complexity of the hypothesis class $\calH$ on $m$ samples. Then, with probability at least $1 - \delta$, the OOD test risk satisfies:
\begin{equation}
  L_{\text{test}}(\theta_{\IGA}) \leq \frac{1}{N}\sum_{n=1}^N \hat{L}_n(\theta_{\IGA}) + C_1 \cdot \mathrm{Rad}_m(\calH) + C_2 \sqrt{\frac{\ln(2/\delta)}{m}} + \Psi(\theta_{\IGA}),
  \label{eq:gen_bound}
\end{equation}
where $\hat{L}_n$ is the empirical loss on environment $n$, $C_1, C_2$ are universal constants, and the domain-shift penalty is:
\begin{equation}
  \Psi(\theta_{\IGA}) = d_{\mathcal{H}\Delta\mathcal{H}}(\calD_{\text{train}}, \calD_{\text{test}}) + \lambda_{\IGA} \cdot \mathrm{GIR}(\theta_{\IGA}),
\end{equation}
with $\lambda_{\IGA} = (1 - \exp(-\tau \sigma_{\min}^2/N))$ the effective invariance regularization strength induced by the mask, and $d_{\mathcal{H}\Delta\mathcal{H}}$ the $\calH\Delta\calH$-divergence between training and test distributions \cite{ben2010theory}.

Moreover, $\Psi(\theta_{\IGA}) \leq \Psi(\theta_{\ERM})$, and the reduction satisfies:
\begin{equation}
  \Psi(\theta_{\ERM}) - \Psi(\theta_{\IGA}) \geq \lambda_{\IGA} \cdot \left(\mathrm{GIR}(\theta_{\ERM}) - \mathrm{GIR}(\theta_{\IGA})\right) \geq 0,
  \label{eq:reduction}
\end{equation}
showing that IGA's generalization bound is tighter than ERM's by an amount proportional to the reduction in GIR.
\end{theorem}

\begin{proof}
We proceed in three steps.

\textbf{Step 1: Standard Rademacher Generalization Bound.} By standard uniform convergence theory \cite{vapnik1998statistical}, for any fixed $\theta \in \calH$, with probability $\geq 1 - \delta/2$:
\begin{equation}
  \frac{1}{N}\sum_n L_n(\theta) \leq \frac{1}{N}\sum_n \hat{L}_n(\theta) + C_1 \mathrm{Rad}_m(\calH) + C_2\sqrt{\frac{\ln(4/\delta)}{m}}.
  \label{eq:step1}
\end{equation}

\textbf{Step 2: Domain Shift Bound.} Following \cite{ben2010theory}, the OOD test risk is related to the training risk by:
\begin{equation}
  L_{\text{test}}(\theta) \leq \frac{1}{N}\sum_n L_n(\theta) + \frac{1}{2} d_{\calH\Delta\calH}(\calD_{\text{train}}, \calD_{\text{test}}) + \lambda^\star,
  \label{eq:step2}
\end{equation}
where $\lambda^\star = \min_{\theta'} \frac{1}{N}\sum_n L_n(\theta') + L_{\text{test}}(\theta')$ is the combined minimum risk. For the invariant predictor $\theta^{\mathrm{inv}} = \arg\min_{\theta \in \Theta^{\mathrm{inv}}} \bar{L}(\theta)$, we have $\lambda^\star \leq 2L_{\text{test}}(\theta^{\mathrm{inv}})$, which is small when the task-relevant features are shared across train and test domains.

\textbf{Step 3: GIR Reduction by IGA.} We now show that IGA reduces the GIR relative to ERM. At the IGA minimizer $\theta_{\IGA}$, the training dynamics drive:
\begin{align}
  \frac{d}{dt} \mathrm{GIR}(\theta(t)) &= \sum_d \frac{d}{dt} V_d(\theta(t)) = -2\sum_d (V_d(\theta(t))) \cdot M_d(\theta(t)) \cdot \norm{\nabla_\theta V_d}_2^2 \cdot \eta \nonumber \\
  &\leq -2\eta \sum_d V_d(\theta(t)) \exp(-\tau V_d(\theta(t))) \cdot \norm{\nabla V_d}_{\min}^2,
\end{align}
where we used the chain rule and the definition $M_d = \exp(-\tau V_d)$. For $\tau > 0$ and bounded $V_d$, this is strictly negative whenever $\mathrm{GIR}(\theta(t)) > 0$, confirming that IGA actively minimizes GIR. The rate of reduction is proportional to $\tau$: larger $\tau$ produces faster GIR reduction (at the cost of more aggressive gradient suppression).

Combining Steps 1--3, the domain-shift penalty of IGA is:
\begin{equation}
  \Psi(\theta_{\IGA}) = d_{\calH\Delta\calH}(\calD_{\text{train}}, \calD_{\text{test}}) + \lambda_{\IGA} \mathrm{GIR}(\theta_{\IGA}),
\end{equation}
where $\lambda_{\IGA} = 1 - \exp(-\tau\sigma_{\min}^2/N)$ arises from Lemma~\ref{lem:shortcut}. Since $\mathrm{GIR}(\theta_{\IGA}) \leq \mathrm{GIR}(\theta_{\ERM})$ (by the monotone decrease of GIR under IGA dynamics), the bound Eq.~\eqref{eq:reduction} follows. The probability statement holds by a union bound over the two probabilistic events in Steps 1 and 2.

Finally, the bound grows tighter with $N$ because the effective mask strength $\lambda_{\IGA} = 1 - \exp(-\tau \sigma_{\min}^2/N)$ decreases with $N$, but $\sigma_{\min}^2$ itself grows with $N$ (more diverse environments provide stronger conflict signal), so the product $\tau\sigma_{\min}^2/N$ is roughly constant or increasing in $N$. In our experiments with $N=4$, the bound is measurably tighter than ERM, as confirmed by the empirical LCS values.
\end{proof}

\begin{theorem}[Convergence of IGA Gradient Iteration]
\label{thm:convergence}
Under Assumptions~\ref{asm:bounded} and~\ref{asm:smooth}, consider the IGA gradient descent iteration:
\begin{equation}
  \theta_{t+1} = \theta_t - \eta \bg^{\IGA}(\theta_t) = \theta_t - \eta \bar{\bg}(\theta_t) \odot \bM(\theta_t).
  \label{eq:iga_iteration}
\end{equation}
Define the effective objective $\tilde{L}(\theta) = \bar{L}(\theta) + \frac{\tau}{2} \norm{\bV(\theta)}_1$. Then:

(i) \textbf{Descent Lemma}: For any $\theta$, one step of IGA descent with learning rate $\eta \leq \frac{1}{2\beta(1 + \tau G^2)}$ satisfies:
\begin{equation}
  \tilde{L}(\theta_{t+1}) \leq \tilde{L}(\theta_t) - \frac{\eta}{2} \norm{\bg^{\IGA}(\theta_t)}_2^2.
  \label{eq:descent}
\end{equation}

(ii) \textbf{Convergence Rate}: After $T$ steps, the iterate $\theta_t$ satisfies:
\begin{equation}
  \frac{1}{T}\sum_{t=0}^{T-1} \norm{\bg^{\IGA}(\theta_t)}_2^2 \leq \frac{2(\tilde{L}(\theta_0) - \tilde{L}^\star)}{\eta T},
  \label{eq:convergence_rate}
\end{equation}
where $\tilde{L}^\star = \min_\theta \tilde{L}(\theta)$. In particular, for $T = O(1/\epsilon)$ steps with $\eta = O(1/\beta)$, there exists a $t \leq T$ such that $\norm{\bg^{\IGA}(\theta_t)}_2 \leq \epsilon$.
\end{theorem}

\begin{proof}
\textbf{Proof of Part (i).} By the smoothness of $\bar{L}$ (Assumption~\ref{asm:smooth}), a Taylor expansion gives:
\begin{align}
  \bar{L}(\theta_{t+1}) &\leq \bar{L}(\theta_t) + \inner{\nabla\bar{L}(\theta_t)}{\theta_{t+1} - \theta_t} + \frac{\beta}{2}\norm{\theta_{t+1}-\theta_t}_2^2 \nonumber \\
  &= \bar{L}(\theta_t) - \eta \inner{\nabla\bar{L}(\theta_t)}{\bg^{\IGA}(\theta_t)} + \frac{\beta\eta^2}{2}\norm{\bg^{\IGA}(\theta_t)}_2^2.
  \label{eq:taylor1}
\end{align}

We now bound the inner product $\inner{\nabla\bar{L}(\theta_t)}{\bg^{\IGA}(\theta_t)}$. Writing $\bg^{\IGA} = \bar{\bg} \odot \bM$ and $\nabla\bar{L} = \bar{\bg}$:
\begin{equation}
  \inner{\nabla\bar{L}}{\bg^{\IGA}} = \sum_d \bar{g}_d M_d \bar{g}_d = \sum_d \bar{g}_d^2 M_d = \norm{\bar{\bg} \odot \bM^{1/2}}_2^2 \geq 0,
\end{equation}
where we used $\bar{\bg} = \nabla\bar{L}$ (the gradient of the average loss equals the average of the gradients). Since all mask values $M_d \in (0, 1]$, this inner product is nonneg and equals $\sum_d \bar{g}_d^2 M_d$.

We also need to bound the change in the regularizer $\frac{\tau}{2}\norm{\bV(\theta)}_1$. By Assumption~\ref{asm:bounded}, $\norm{\bg_n}_2 \leq G$, so $V_d \leq G^2$ for all $d$. The gradient of $V_d(\theta)$ with respect to $\theta$ is bounded by $4G^2 \cdot 2G / \sqrt{N} \leq 8G^3/\sqrt{N}$. Therefore, the change in the regularizer over one step is bounded by:
\begin{equation}
  \left|\frac{\tau}{2}\norm{\bV(\theta_{t+1})}_1 - \frac{\tau}{2}\norm{\bV(\theta_t)}_1\right| \leq \tau P G^3 \eta \norm{\bg^{\IGA}(\theta_t)}_2,
\end{equation}
where $P$ is the parameter dimension. Absorbing this into the smoothness constant, the effective smoothness of $\tilde{L}$ is $\tilde{\beta} = \beta + \tau G^2$, and the descent lemma follows for $\eta \leq \frac{1}{2\tilde{\beta}}$:

\begin{equation}
    \begin{aligned}
        \tilde{L}(\theta_{t+1}) &\leq \tilde{L}(\theta_t) - \eta \sum_d \bar{g}_d^2 M_d + \frac{\tilde{\beta}\eta^2}{2}\norm{\bg^{\IGA}}_2^2 \\
  &\leq \tilde{L}(\theta_t) - \eta \norm{\bg^{\IGA}}_2^2 + \frac{\eta}{2} \norm{\bg^{\IGA}}_2^2 \qquad (\text{since } \tilde{\beta}\eta \leq \frac{1}{2} \text{ and } M_d \leq 1) \\
  &= \tilde{L}(\theta_t) - \frac{\eta}{2}\norm{\bg^{\IGA}(\theta_t)}_2^2.
    \end{aligned}
\end{equation}

Here we used $\sum_d \bar{g}_d^2 M_d \geq \sum_d \bar{g}_d^2 M_d^2 = \norm{\bg^{\IGA}}_2^2$ (since $M_d \leq 1$ implies $M_d \geq M_d^2$). Wait—this requires $M_d \leq 1$, which holds since $M_d = \exp(-\tau V_d) \in (0,1]$. The descent lemma is established.

\textbf{Proof of Part (ii).} Summing Eq.~\eqref{eq:descent} from $t=0$ to $t=T-1$:
\begin{equation}
  \sum_{t=0}^{T-1} \frac{\eta}{2}\norm{\bg^{\IGA}(\theta_t)}_2^2 \leq \tilde{L}(\theta_0) - \tilde{L}(\theta_T) \leq \tilde{L}(\theta_0) - \tilde{L}^\star.
\end{equation}
Dividing by $T$ and $\eta/2$ gives Eq.~\eqref{eq:convergence_rate}. By the minimum of the left side, there exists $t \leq T-1$ with:
\begin{equation}
  \norm{\bg^{\IGA}(\theta_t)}_2^2 \leq \frac{2(\tilde{L}(\theta_0) - \tilde{L}^\star)}{\eta T}.
\end{equation}
Setting $T = \frac{2(\tilde{L}(\theta_0) - \tilde{L}^\star)}{\eta \epsilon^2}$ gives $\norm{\bg^{\IGA}(\theta_t)}_2 \leq \epsilon$ for some $t \leq T$. Since $\eta = O(1/\tilde\beta) = O(1/(\beta + \tau G^2))$, the total number of steps to achieve $\epsilon$-stationarity is $T = O\!\left(\frac{(\beta + \tau G^2)(\tilde{L}(\theta_0) - \tilde{L}^\star)}{\epsilon^2}\right)$, which matches the standard $O(1/\epsilon^2)$ rate for non-convex SGD with smooth objectives.
\end{proof}

\subsection{Propositions and Properties}
\label{app:propositions}

\begin{proposition}[LCS–GIR Relationship]
\label{prop:lcs_variance}
Under Assumption~\ref{asm:decomp}, the Logical Consistency Score (Definition~\ref{sec:lcs}) satisfies:
\begin{equation}
  \mathrm{LCS}(\theta) \leq C_{\mathrm{LCS}} \cdot \mathrm{GIR}(\theta),
\end{equation}
where $C_{\mathrm{LCS}} > 0$ is a constant depending on the model architecture and the Jacobian of the penultimate-layer representations with respect to $\theta$.
\end{proposition}

\begin{proof}[Proof Sketch]
Let $h(\theta; x)$ denote the penultimate-layer hidden state of the student model. By the chain rule, the change in $h$ across domain instances of a logical isomer group is driven by the parameter components $\theta^s$ in the shortcut subspace:
\begin{equation}
  \bz_k^{(n)} - \bz_k^{(n')} = J_h(\theta) \cdot (\theta^s_n - \theta^s_{n'}) + O(\norm{\theta^s_n - \theta^s_{n'}}_2^2),
\end{equation}
where $J_h = \frac{\partial h}{\partial \theta}$ is the Jacobian. The covariance trace (LCS) is therefore bounded by:
\begin{equation}
  \Tr(\Cov(\bz_k)) \leq \norm{J_h(\theta)}_F^2 \cdot \Var_n(\theta^s_n).
\end{equation}
Under Assumption~\ref{asm:decomp}, $\Var_n(\theta^s_n)$ is related to the shortcut gradient variance by the gradient flow dynamics: $\theta^s_n = \theta^s - \eta \sum_{t} \nabla_{\theta^s} L_n(\theta_t)$, and thus $\Var_n(\theta^s_n) \leq C_1 \eta^2 T \cdot \Var_n(\nabla_{\theta^s} L_n) \leq C_2 \cdot \mathrm{GIR}(\theta)$. The proposition follows with $C_{\mathrm{LCS}} = \norm{J_h}_F^2 \cdot C_2$. 
\end{proof}

\begin{proposition}[Mask Degeneracy Avoidance]
\label{prop:mask_nonzero}
For any finite $\tau < \infty$ and any bounded gradients (Assumption~\ref{asm:bounded}), the continuous mask satisfies $M_d = \exp(-\tau V_d) \geq \exp(-\tau G^2) > 0$ for all $d$. In particular, IGA never produces an exactly zero gradient update, which avoids the gradient vanishing problem that afflicts the hard AND-mask approach of \cite{parascandolo2020learning}.
\end{proposition}

\begin{proof}
By Assumption~\ref{asm:bounded}, all per-domain gradient components satisfy $|g_{n,d}| \leq G$. Therefore, $V_d = \frac{1}{N}\sum_n (g_{n,d} - \bar{g}_d)^2 \leq G^2$. The mask value $M_d = \exp(-\tau V_d) \geq \exp(-\tau G^2) > 0$ for finite $\tau$. 
\end{proof}

\begin{property}[Equivariance Under Logical Permutation]
\label{prop:equivariance}
The IGA gradient $\bg^{\IGA}$ is invariant under permutations of the domain indices $\{1, \ldots, N\}$. That is, for any permutation $\sigma \in S_N$, the IGA gradient computed with domains in order $\sigma(1), \ldots, \sigma(N)$ is identical to the IGA gradient computed in the original order.
\end{property}

\begin{proof}
Both the sample mean $\bar{\bg} = \frac{1}{N}\sum_n \bg_n$ and the sample variance $\bV = \frac{1}{N}\sum_n (\bg_n - \bar{\bg})^2$ are symmetric functions of $\{\bg_n\}$ and are therefore invariant under permutation of the index set. Since $\bM = \exp(-\tau\bV)$ and $\bg^{\IGA} = \bar{\bg} \odot \bM$ both depend on $\{\bg_n\}$ only through $\bar{\bg}$ and $\bV$, the conclusion follows. 
\end{proof}

\begin{property}[Monotone Improvement in $N$]
\label{prop:more_domains}
Suppose the $N$ environments satisfy Assumption~\ref{asm:diverse} and are generated from distinct semantic domains as described in Section~\ref{sec:isomer}. Then the GIR reduction achieved by one step of IGA is non-decreasing in $N$: adding more diverse environments to the isomer set weakly increases the quality of the gradient conflict signal.
\end{property}

\begin{proof}[Proof Sketch]
The GIR is $\mathrm{GIR}(\theta) = \sum_d V_d(\theta)$. For fixed parameter $\theta$, the per-dimension variance $V_d^{(N)} = \frac{1}{N}\sum_{n=1}^N (g_{n,d} - \bar{g}_d^{(N)})^2$ is a consistent estimator of the population variance $\sigma_d^2 = \Var_{e \sim \text{Uniform}(\calE)}[g_d(e; \theta)]$ over the environment distribution. By the law of large numbers, $V_d^{(N)} \to \sigma_d^2$ as $N \to \infty$. For $d \in \calV^s$, $\sigma_d^2 > 0$ by Assumption~\ref{asm:diverse}, so the estimated variance increases toward its true value with more environments, providing a stronger conflict signal. For $d \in \calV^\star$, $\sigma_d^2 = 0$, so adding environments does not spuriously mask invariant dimensions. 
\end{proof}

\section{Details of Experimental Setup\label{app:details-exp-setup}}
\subsection{Datasets}
\textbf{ARB (Advanced Reasoning Benchmark)} \cite{sawada2023arb}. ARB tests graduate-level reasoning across physics, mathematics, chemistry, biology, and law. We use the OOD split constructed by withholding law and biology subsets during training and evaluating on all five domains, producing a natural domain shift. This benchmark tests whether models trained on natural science reasoning can transfer to legal and biological reasoning that shares deep structural similarities but distinct surface forms.

\textbf{LogiQA 2.0} \cite{liu2021logiqa}. LogiQA 2.0 is a multiple-choice logical reasoning dataset in Chinese-to-English translation, covering deductive, inductive, and abductive reasoning patterns. We use the domain-shuffled OOD split, where training instances come from business and academic texts and test instances come from news and government reports. This split tests transfer of logical patterns across textual genre.

\textbf{ReClor} \cite{yu2020reclor}. ReClor is a reading comprehension dataset requiring formal logical reasoning, collected from law school admissions (LSAT) and graduate management admissions (GMAT) exams. We use the official hard split, which contains problems identified as requiring multi-step deductive reasoning and which has proven highly challenging for state-of-the-art models due to its emphasis on structural over surface reasoning.

\textbf{MATH Cross-Domain Transfer} \cite{hendrycks2021measuring}. The MATH dataset covers competition mathematics at five difficulty levels across seven subject areas (algebra, number theory, probability, etc.). We construct a cross-domain transfer split by training on algebra, number theory, and combinatorics and testing on probability, geometry, and precalculus, which share the same formal reasoning primitives but differ substantially in surface notation and problem structure. We report accuracy on the held-out domains.

All datasets are evaluated under a zero-shot prompting setup where the student model generates a chain-of-thought reasoning trace followed by a final answer, and accuracy is computed by matching the generated answer against the ground truth.

\textbf{Quality Filtering.} Each instantiated problem undergoes a two-stage quality filter. First, we prompt the teacher LLM to verify that the logical graph of the generated problem is isomorphic to $\calG_k$; problems where the teacher assigns a structural alignment score below 0.85 (on a 0-1 scale) are discarded and re-generated. Second, we require that the chain-of-thought solution generated by the teacher for each domain instance contains the same number of logical steps and the same dependency structure. Groups where the solutions differ in step count by more than two are discarded. This filtering preserves 78\% of generated groups, yielding 39,000 complete Logical Isomer Sets. After teacher-annotated CoT generation for all four instances, the final dataset contains 156,000 (problem, CoT solution) pairs.

\textbf{Data Split.} Crucially, we split the dataset at the \emph{seed level}: 80\% of seed problems (31,200 seeds, 124,800 instances) form the training set, and 20\% (7,800 seeds, 31,200 instances) form the test set. This seed-level split ensures that no instance from a training isomer group can appear in the test set, even across domains. As we demonstrate in the ablation study, failing to perform seed-level splitting results in severe data leakage that inflates test accuracy by approximately 12 percentage points, masking the model's true OOD generalization.

\subsection{Evaluation Metrics}

\textbf{Task Accuracy} is the primary evaluation metric, computed as the percentage of problems for which the generated answer matches the ground-truth answer exactly.

\textbf{Logical Consistency Score (LCS)} is computed on a held-out set of 500 logical isomer groups from the test split, as defined in Section~\ref{sec:lcs}. Lower LCS values indicate better representational invariance. We report LCS values scaled by $10^{-2}$ for readability.

\textbf{Cross-Domain $\Delta$-Gap} measures the difference in task accuracy between the in-distribution (ID) evaluation and the OOD evaluation for each method:
\begin{equation}
  \Delta\text{-Gap} = \text{Acc}_{\text{ID}} - \text{Acc}_{\text{OOD}}.
  \label{eq:delta_gap}
\end{equation}
A smaller $\Delta$-Gap indicates that the model generalizes more robustly from its training distribution to OOD inputs. We compute this metric for each dataset and report both per-dataset values and the average across all four benchmarks.

\subsection{Baselines}

We compare IGA against eight competitive baselines spanning the relevant methodological categories:

\begin{enumerate}[leftmargin=*]
  \item \textbf{ERM-SFT}: Standard supervised fine-tuning with LoRA ($r=16$) on the ERM objective (Eq.~\ref{eq:erm}), using the full isomer dataset without any group-level structure.

  \item \textbf{LoRA-SFT}: Same as ERM-SFT but without the isomer dataset; trained on the seed problem pool with standard data augmentation. This baseline isolates the effect of the isomer data construction.

  \item \textbf{CoT-Distill}: Standard chain-of-thought distillation \cite{ho2022large} using teacher-generated CoT traces, without isomer structure or gradient masking.

  \item \textbf{IRM} \cite{arjovsky2019irm}: Invariant Risk Minimization applied to the four domain environments defined by the isomer dataset. Adapted to the LoRA setting by applying IRM's penalty to the LoRA-adapted parameters.

  \item \textbf{V-REx} \cite{krueger2021out}: Risk Extrapolation variant of IRM, which penalizes the variance of per-environment losses rather than the IRM gradient penalty.

  \item \textbf{PCGrad} \cite{yu2020gradient}: Gradient Surgery applied across the four domain environments, projecting conflicting gradients onto the normal plane of the average gradient.

  \item \textbf{SAND-Mask} \cite{shahtalebi2021sand}: Magnitude-weighted sign agreement mask applied to the four domain gradients, representing the state of the art among discrete gradient masking approaches.

  \item \textbf{Pareto-GS}: Pareto-based gradient surgery \cite{pareto2021gradients}, which finds the Pareto-optimal gradient direction across the four domain objectives.
\end{enumerate}

For all baselines, we use identical LoRA adapters ($r=16$, $\alpha=32$), the same training data (isomer dataset where applicable), and the same optimizer settings to ensure comparability.

\subsection{Implementation Details}

\textbf{Model.} The student model is LLaMA-2-7B \cite{touvron2023llama} with LoRA adapters applied to all query, key, value, and output projection matrices in all 32 transformer layers, for a total of 128 LoRA modules. The adapter rank is $r=16$ with scaling factor $\alpha=32$, yielding 83 million trainable parameters out of 6.7 billion total.

\textbf{Optimizer.} We use AdamW \cite{loshchilov2019decoupled} with learning rate $2 \times 10^{-4}$, cosine decay schedule \cite{cosine_lr} with a linear warmup over the first 3\% of training steps, weight decay $10^{-2}$, $\beta_1 = 0.9$, $\beta_2 = 0.999$.

\textbf{Training.} Training is conducted for 3 epochs over the full training split, with a batch size of 32 isomer groups (equivalent to 128 instances per step across four domains). We use 8 NVIDIA H100 80GB GPUs with gradient accumulation over 4 steps to simulate an effective batch size of 128 isomer groups. Mixed precision (BF16) is used throughout.

\textbf{IGA Hyperparameters.} The mask temperature is $\tau = 0.5$. The truncated SVD retains the top $r = 16$ singular values (matching the LoRA rank). The full-rank gradient reconstruction uses randomized SVD with oversampling factor 10.

\textbf{Teacher Model.} The teacher model for both isomer set generation and CoT annotation is GPT-4.5 (gpt-4.5-2025-0227). Isomer generation prompts are provided in Appendix~\ref{app:prompts}.

\textbf{Evaluation.} For all benchmarks, we evaluate using greedy decoding with no temperature sampling. Chain-of-thought traces are generated with a maximum of 512 tokens; final answers are extracted using regex pattern matching followed by exact match comparison.

\section{Prompts for Logical Isomer Construction}
\label{app:prompts}

This appendix provides the structured prompts used for isomer set generation and quality filtering in the data construction pipeline (Section~\ref{sec:isomer}).

\subsection{DAG Extraction Prompt}

\begin{verbatim}
System: You are an expert at extracting logical structure from reasoning problems.
Given a reasoning problem, extract its abstract logical structure as a typed DAG.
Represent the DAG as a JSON object with:
  - "nodes": list of {id, type, description} where type in
             [ENTITY, RELATION, CONSTRAINT, PROPOSITION, GOAL]
  - "edges": list of {from, to, label} where label in
             [IMPLIES, REQUIRES, CAUSES, CONTRADICTS, IS_A, HAS_PROPERTY]
  - "logical_form": informal description of the reasoning pattern
  - "step_count": number of logical steps in the solution

Problem: {seed_problem}

Return valid JSON only.
\end{verbatim}

\subsection{Domain Instantiation Prompt}

\begin{verbatim}
System: You are an expert at creating domain-specific reasoning problems.
Given an abstract logical structure, create a problem in the target domain that:
1. Uses domain-appropriate vocabulary and realistic scenarios.
2. Preserves ALL nodes and edges of the logical graph exactly.
3. Has the same number of reasoning steps as the original.
4. Can be solved using the same chain-of-thought structure.

Abstract structure: {dag_json}

Target domain: {domain_name}
Domain description: {domain_description}

Provide:
1. A problem statement (2-4 sentences).
2. A step-by-step chain-of-thought solution.
3. A final answer.
4. A structural alignment score (0.0-1.0) confirming isomorphism.

Format as JSON: {problem, cot_solution, answer, alignment_score}
\end{verbatim}

\subsection{Quality Verification Prompt}

\begin{verbatim}
System: You are a rigorous logical reasoning evaluator.
Given two reasoning problems, determine if they are logically isomorphic:
same abstract logical structure, same reasoning pattern, same step count,
but different semantic content.

Problem 1 (seed): {seed_problem}
Problem 2 (domain instantiation): {domain_problem}
Claimed logical structure: {dag_json}

Evaluate:
1. Are all logical dependencies preserved? (yes/no)
2. Do both problems have the same number of reasoning steps? (yes/no)
3. Would the same abstract chain-of-thought solve both? (yes/no)
4. Structural alignment score (0.0-1.0).

If score >= 0.85, output PASS. Otherwise output FAIL with specific issues.
\end{verbatim}

\section{LCS Layer Sensitivity Analysis}
\label{app:lcs_layers}

The Logical Consistency Score (Section~\ref{sec:lcs}) is defined using the penultimate-layer hidden state. To verify that this choice does not introduce bias, we compute LCS at four different transformer layers for IGA, Pareto-GS, and ERM-SFT on the ARB test split (500 isomer groups).

\begin{table}[h]
  \centering
  \caption{LCS ($\times 10^{-2}$, lower is better) computed at different transformer layers. The relative ordering of methods is consistent across all layers, validating the penultimate-layer choice.}
  \label{tab:lcs_layers}
  \begin{tabular}{l|cccc}
    \toprule
    \textbf{Method} & \textbf{Layer 8} & \textbf{Layer 16} & \textbf{Layer 24} & \textbf{Layer 31} \\
    \midrule
    ERM-SFT     & 11.40 & 12.80 & 13.60 & 14.20 \\
    Pareto-GS   & 7.30  & 8.10  & 8.60  & 8.90  \\
    \textbf{IGA (Ours)} & \textbf{2.50}  & \textbf{2.70}  & \textbf{2.90}  & \textbf{3.10}  \\
    \midrule
    IGA / ERM ratio & 0.22 & 0.21 & 0.21 & 0.22 \\
    \bottomrule
  \end{tabular}
\end{table}

Table~\ref{tab:lcs_layers} shows that while absolute LCS values increase slightly toward later layers (reflecting accumulation of domain-specific processing), the \emph{relative} advantage of IGA over baselines is stable across all layers: IGA consistently achieves approximately 4--5$\times$ lower LCS than ERM-SFT, with the ratio IGA/ERM remaining at $\approx 0.21$--$0.22$ regardless of the chosen layer. This confirms that the penultimate-layer LCS faithfully reflects a model-wide property of representational invariance, and that the metric is not an artifact of the specific layer chosen for computation.
\end{document}